\documentclass[letterpaper]{article} 
\usepackage{aaai2026}  
\usepackage{times}  
\usepackage{helvet}  
\usepackage{courier}  
\usepackage[hyphens]{url}  
\usepackage{graphicx} 
\usepackage{natbib}  
\usepackage{caption} 
\usepackage{graphicx}
\usepackage[table]{xcolor}
\usepackage{xspace}
\usepackage{amsmath}
\usepackage{amssymb}
\usepackage[labelformat=simple]{subcaption}
\usepackage{algorithm}
\usepackage{algpseudocode}
\usepackage{multirow} 
\usepackage{cuted}
\usepackage{afterpage}
\usepackage{hhline}
\usepackage{tabularx}
\usepackage{amsthm}
\usepackage{enumitem}
\usepackage[skip=4pt]{caption}
\DeclareMathOperator*{\argmin}{argmin}
\usepackage{pifont}
\usepackage{booktabs}
\usepackage{kotex}
\usepackage{makecell}

\newcommand{\realset}{\mathbb{R}}

\newcommand{\mat}[1]{\mathbf{#1}}

\newcommand{\loss}{\mathcal{L}}

\definecolor{ForestGreen}{rgb}{0.133, 0.545, 0.133}
\definecolor{torange}{rgb}{0.949, 0.522, 0.}
\definecolor{BrickRed}{rgb}{0.67, 0.29, 0.26}

\newcommand{\gcmark}{\textcolor{ForestGreen}{\ding{51}}}

\newcommand{\rxmark}{\textcolor{BrickRed}{\ding{55}}}

\newcommand{\method}{\textsc{LampQ}\xspace}
\newcommand{\methodfullbold}{\textbf{\underline{La}}yer-wise \textbf{\underline{M}}ixed \textbf{\underline{P}}recision \textbf{\underline{Q}}uantization for Vision Transformers\xspace}

\newcommand{\titletext}{\method: Towards Accurate Layer-wise Mixed Precision Quantization \\ for Vision Transformers\xspace}

\newtheorem{problem}{Problem}

\newtheorem{lemma}{Lemma}

\newcolumntype{Y}{>{\centering\arraybackslash}X}
\newtheoremstyle{dotless}{}{}{\itshape}{}{\bfseries}{}{ }{}

\theoremstyle{dotless}

\newcommand{\MP}{\tiny{\textsc{MP}}\normalsize\xspace}

\newcommand{\fpmodel}{f_{\theta}\xspace}
\newcommand{\quantparam}{\theta'\xspace}
\newcommand{\quantmodel}{f_{\theta'}\xspace}
\newcommand{\matrixfp}{\mat{M}\xspace}
\newcommand{\matrixquant}{\mat{M}'\space}
\newcommand{\matrixqdq}{\widehat{\mat{M}}}
\newcommand{\weightifp}{\mat{W}_i\xspace}
\newcommand{\actifp}{\mat{X}_i\xspace}

\newcommand{\weightiqdq}{\widehat{\mat{W}}_i}
\newcommand{\actiqdq}{\widehat{\mat{X}}_i}

\newcommand{\hessian}{\mat{H}\xspace}
\newcommand{\hessiani}{\mat{H}_i\xspace}
\newcommand{\hessianj}{\mat{H}_j\xspace}

\newcommand{\deltawi}{\Delta_{\mat{W}_i}}
\newcommand{\deltawj}{\Delta_{\mat{W}_j}}
\newcommand{\deltaxi}{\Delta_{\mat{X}_i}}
\newcommand{\deltaw}{\Delta_W}
\newcommand{\deltax}{\Delta_X}

\newcommand{\imageI}{\mathcal{I}}

\newcommand{\msa}{\text{MSA}\xspace}
\newcommand{\mlp}{\text{MLP}\xspace}
\newcommand{\qkv}{\textit{qkv}\xspace}
\newcommand{\proj}{\textit{proj}\xspace}
\newcommand{\fcone}{\textit{fc1}\xspace}
\newcommand{\fctwo}{\textit{fc2}\xspace}

\newcommand{\trace}{\textnormal{tr}}
\newcommand{\fisher}{\mat{F}}
\newcommand{\fisheri}{\mat{F}_i}

\newcommand{\smallsection}[1]{\noindent\textbf{#1.}}

\usepackage{newfloat}
\usepackage{listings}
\DeclareCaptionStyle{ruled}{labelfont=normalfont,labelsep=colon,strut=off} 
\floatstyle{ruled}
\newfloat{listing}{tb}{lst}{}
\floatname{listing}{Listing}
\title{\titletext}
\author {
    Minjun Kim\textsuperscript{\rm 1}, 
    Jaeri Lee\textsuperscript{\rm 2},
    Jongjin Kim\textsuperscript{\rm 1},
    Jeongin Yun\textsuperscript{\rm 2},
    Yongmo Kwon\textsuperscript{\rm 2},
    U Kang\textsuperscript{\rm 1, 2}\thanks{ Corresponding author.},
}
\affiliations {
    \textsuperscript{\rm 1}Department of Computer Science and Engineering, Seoul National University \\
    \textsuperscript{\rm 2}Interdisciplinary Program in Artificial Intelligence, Seoul National University \\
    \{minjun.kim, jlunits2, j2kim99, yji00828, rnjsdydah, ukang\}@snu.ac.kr
}

\graphicspath{ {images/} }

\begin{document}

\maketitle

\begin{abstract}
    How can we accurately quantize a pre-trained Vision Transformer model?
Quantization algorithms compress Vision Transformers (ViTs) into low-bit formats, reducing memory and computation demands with minimal accuracy degradation.
However, existing methods rely on uniform precision, ignoring the diverse sensitivity of ViT components to quantization.
Metric-based Mixed Precision Quantization (MPQ) is a promising alternative, but previous MPQ methods for ViTs suffer from three major limitations: 1) coarse granularity, 2) mismatch in metric scale across component types, and 3) quantization-unaware bit allocation.
In this paper, we propose \textbf{\method} (\methodfullbold), an accurate
metric-based MPQ method for ViTs to overcome these limitations.
\method performs layer-wise quantization to achieve both fine-grained control and efficient acceleration, incorporating a type-aware Fisher-based metric to measure sensitivity.
Then, \method assigns bit-widths optimally through integer linear programming and further updates them iteratively.
Extensive experiments show that \method provides the state-of-the-art performance in quantizing ViTs pre-trained on various tasks such as image classification, object detection, and zero-shot quantization. 
\end{abstract}

\begin{links}
	\link{Code}{https://github.com/snudatalab/LampQ}
\end{links}

\section{Introduction}
\label{sec:intro}
\textit{How can we compress a pre-trained Vision Transformer model while maintaining accuracy?}
Vision Transformers (ViTs)~\cite{ViT, DeiT} have recently gained significant attention due to their superior performance across a wide range of computer vision problems
~\cite{ViTObjectDetection, SimVP, LLAVA}.
%
%
Despite their success, ViTs are difficult to deploy on resource-limited devices due to their complex architecture, along with significant memory and computational demands~\cite{Efficientformer, EfficientViT}.
Model quantization~\cite{SynQ, ZSQSurvey} mitigates these challenges by converting models into a low-bit format, which enables higher compression rate and faster inference with minimal performance loss over other compression methods such as pruning~\cite{KPrune, SPRINT}, knowledge distillation~\cite{Mustard, Pea-KD, PET}, and low-rank approximation~\citep{Falcon}.
Among the two approaches of quantization, Post-Training Quantization (PTQ)~\cite{ERQ, UniQuanF} is more suitable for ViTs since quantization-aware training~\cite{Q-ViT, GPUSQ-ViT} requires training that may take up to several days or weeks.

\begin{figure}[t]
	\centering
	\includegraphics[width=\linewidth]{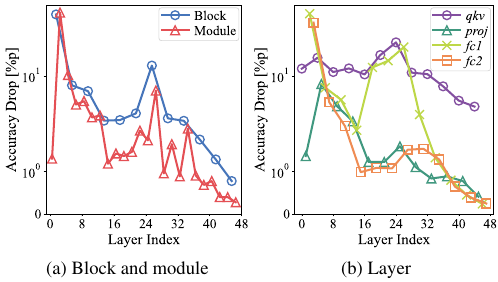}
	\caption{
		\label{fig:granularity}
		Accuracy when quantizing a single component of the DeiT-S model to 1-bit following~\citet{AdaLog} while keeping the others unchanged.
		Sensitivity varies significantly across (a) blocks and modules, and (b) layers.
	}
\end{figure}


Previous PTQ studies~\cite{RepQ-ViT, PD-Quant, IGQ-ViT} have achieved promising performance by tackling ViT-specific problems such as inter-channel variation of post-LayerNorm activations~\cite{RepQ-ViT}, power-law distribution of post-Softmax activations~\cite{AdaLog}, and outliers of block outputs~\cite{OAS}.
However, most of these methods employ uniform precision, applying the same bit-width across all components.
This approach leads to suboptimal performance since different ViT components (such as blocks, modules, and layers) exhibit varying sensitivities to quantization (see Figures~\ref{fig:granularity},~\ref{fig:vit}, and~\ref{fig:iterative}).

%
Assigning different quantization bit-widths to each component based on their sensitivity, known as Mixed Precision Quantization (MPQ)~\cite{OneShotMPQ, MPSurvey}, aims to improve performance by allocating higher bit-widths to critical network layers.
Metric-based methods~\cite{SensiMix, QE, OMPQ} emerge as the most practical choice among the MPQ techniques.
Other MPQ techniques such as
gradient~\cite{SDQ}, reinforcement learning~\cite{MetaMix}, and neural architecture search~\cite{JAQ},
lack scalability as they demand substantial computational resources; they take 10 to 300 GPU hours even for a small model such as ResNet-18 with only 11M parameters~\cite{MetaMix}.
%


\begin{figure}[t]
	\centering
	\includegraphics[width=0.95\linewidth]{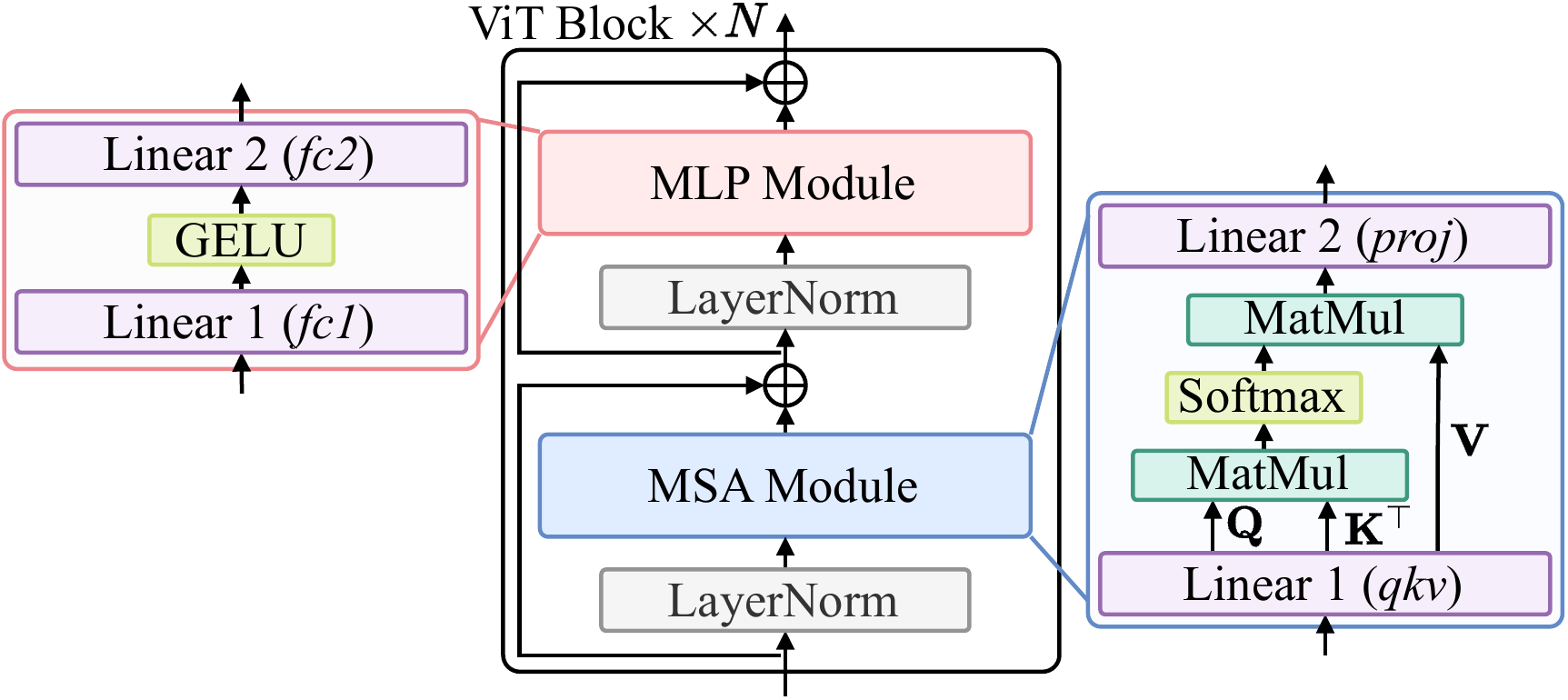}
	\caption{
		Illustration of a ViT model with $N$ blocks.
		Each block consists of two modules: \msa (red) and \mlp (blue),
    	and four layers: \qkv, \proj, \fcone, and \fctwo (purple).
	}
	\label{fig:vit}
\end{figure}


While metric-based MPQ has a potential to be an optimal solution, the only existing MPQ method for ViTs~\cite{VT-PTQ} exhibits limited performance due to its simplistic design of the following key components (see Table~\ref{tab:components}):
\begin{itemize}[leftmargin=3mm, topsep=1mm, itemsep=0mm]
	\item{
		\textbf{Granularity.}
		A module-wise MPQ approach applies the same bit-width across an entire module, disregarding the varying sensitivity of different layers (see Figure~\ref{fig:granularity}(b)).
		This coarse-grained granularity results in suboptimal bit allocation, harming layers that require finer precision.
	}
	\item{
		\textbf{Metric.}
		The nuclear norms of attention and feature maps vary significantly in scale depending on the module type, such as \msa and \mlp (see Figure~\ref{fig:metric}(a)).
        This discrepancy hinders direct comparison of metrics across modules, leading to inaccurate sensitivity assessments.
	}
	\item{
		\textbf{Bit assignment.}
			The Pareto frontier approach in prior quantization works is costly and relies on a fixed metric, which fails to reflect error changes during quantization and results in suboptimal decisions (see Figure~\ref{fig:iterative}).
	}
\end{itemize}

In this paper, we propose \textbf{\method} (\methodfullbold), an accurate PTQ method for ViTs by a layer-wise metric-based MPQ.
By choosing layers as its quantization granularity, \method ensures both fine-grained control and efficient acceleration via kernel support.
\method quantifies each layer’s sensitivity with a type-aware Fisher-based metric, enabling direct comparison between layers with different types.
\method assigns bits initially through integer linear programming and then updates them iteratively to dynamically reflect changes in sensitivity.
\method is powerful and versatile since it is easily integrated with any PTQ method for ViTs, achieving the state-of-the-art performance.

\def\arraystretch{0.7}
\begin{table}[t]
	\centering
	\caption{
		Comparison of the key components in metric-based MPQ between VT-PTQ~\cite{VT-PTQ} and \method.
	}
		\setlength{\tabcolsep}{3pt}
        \small
		\begin{tabular}{cccc}
			\hline
			\toprule
			\textbf{Component} & \makecell{\textbf{VT-PTQ}} & \multicolumn{2}{c}{\makecell{\textbf{\method} \textbf{(Proposed)}}}\\
			\midrule
			\multirow{2}{*}{\textbf{Granularity}} & \multirow{2}{*}{Module-wise} & \multicolumn{2}{c}{\multirow{2}{*}{Layer-wise}}\\
			& & & \\
			\midrule
			\textbf{Metric} & \makecell{Nuclear norm of \\ attention / feature maps} &  \makecell{Trace of \\ Fisher info.} & \makecell{Layer-wise \\ recon. error}\\
			\midrule
			\textbf{Bit assign.} & \makecell{Pareto frontier} & \makecell{Integer linear \\ programming} & \makecell{Iterative \\ bit update} \\
			\bottomrule
			\hline
		\end{tabular}
	\label{tab:components}
\end{table}


Our main contributions are summarized as follows:
\begin{itemize}[leftmargin=3.2mm, itemsep=0mm, topsep=0mm]
    \item{
        \textbf{Observation.}
		We observe three major challenges in designing a metric-based MPQ method for ViTs each associated with a specific component: granularity, metric, and bit assignment (see Figures~\ref{fig:granularity},~\ref{fig:metric}, and~\ref{fig:iterative}).
	}
    \item{
    	\textbf{Method.}
		We propose \method, an accurate MPQ method for ViTs.
		\method carefully assigns quantization bits to achieve both efficiency and accuracy (see Figure~\ref{fig:method}).
		\method considers the differences between types of layers by a type-aware Fisher-based metric,
		and incorporates the quantization feedback by iterative bit updates.
    }
    \item{
        \textbf{Experiments.}
        We experimentally show that \method consistently outperforms its competitors on various models and datasets in image classification, object detection, and zero-shot quantization tasks (see Tables~\ref{tab:q1}, \ref{tab:q2}, and \ref{tab:q3}).
    }
\end{itemize}

%

\section{Preliminaries}
\label{sec:prelim}

\subsection{Problem Definition}
\label{subsec:2_problem}

Given a pre-trained model, a small calibration dataset, and quantization bits, Post-Training Quantization (PTQ) targets to optimize the quantized model to maintain performance.
%
\begin{problem}[Post-training Quantization for Vision Transformers (ViTs)~\cite{ERQ, AdaLog}]
	{\color{white}.}
	\label{problem}
	\begin{itemize}[leftmargin=3.2mm, topsep=0mm]
		\item{
			\textbf{Input}:
			a ViT model $\fpmodel$ with parameters $\theta$ pre-trained on a target task $\mathcal{T}$,
			a sample dataset $\mathbb{D}=\{(\mathbf{x}_i, y_i)\}_{i=1}^S$ of size $S$, and
			quantization bits $b_t$.
		}
		\item{
			\textbf{Output}:
			a quantized model $\quantmodel$ with parameters $\quantparam$ within the $b_t$-bit limit minimizing performance degradation on $\mathcal{T}$.
		}
	\end{itemize}
\end{problem}

\subsection{Vision Transformers}
\label{subsec:2_vit}

ViTs~\cite{ViT} are deep learning models that apply self-attention to capture image context and improve feature representation.
Figure~\ref{fig:vit} illustrates the architecture of a standard ViT block, which we reformulate in a simplified manner.
%
%
%
%
Each block consists of two modules: Multi-head Self-Attention (MSA) and Multi-Layer Perceptron (MLP).

\smallsection{Simplified Formulation}
To facilitate layer-wise quantization analysis, we abstract each ViT block as consisting of four core linear layers: $\qkv$, $\proj$, $\fcone$, and $\fctwo$.
While this formulation omits the inner workings of the attention mechanism, we empirically validate its effectiveness later
(see Section \textit{`Type-aware Fisher-based Metric'}).
This abstraction supports a metric both principled and effective for ViTs.

\subsection{Model Quantization}
\label{subsec:2_quant}

Model quantization involves converting a model into a lower bit precision.
We focus on asymmetric uniform quantization, following previous researches~\cite{RepQ-ViT, AdaLog}.
Given a matrix $\matrixfp$, the $B$-bit quantized matrix $\matrixquant = \left\lfloor \matrixfp/s - z + 0.5 \right\rfloor$,
where $s=(r_{max}-r_{min})/(2^{B}-1)$ is the scaling factor,
$z=r_{min}/s + 2^{B-1}$ is the integer offset,
and $(r_{min}, r_{max})$ are the lower and upper bounds of $\matrixfp$.
The corresponding dequantized value is given by $\matrixqdq= s(\matrixquant+z)$.
While some methods quantize only the weights, we quantize both weights and activations
for efficient inference.

\begin{figure}[t]
	\centering
	\includegraphics[width=0.95\columnwidth]{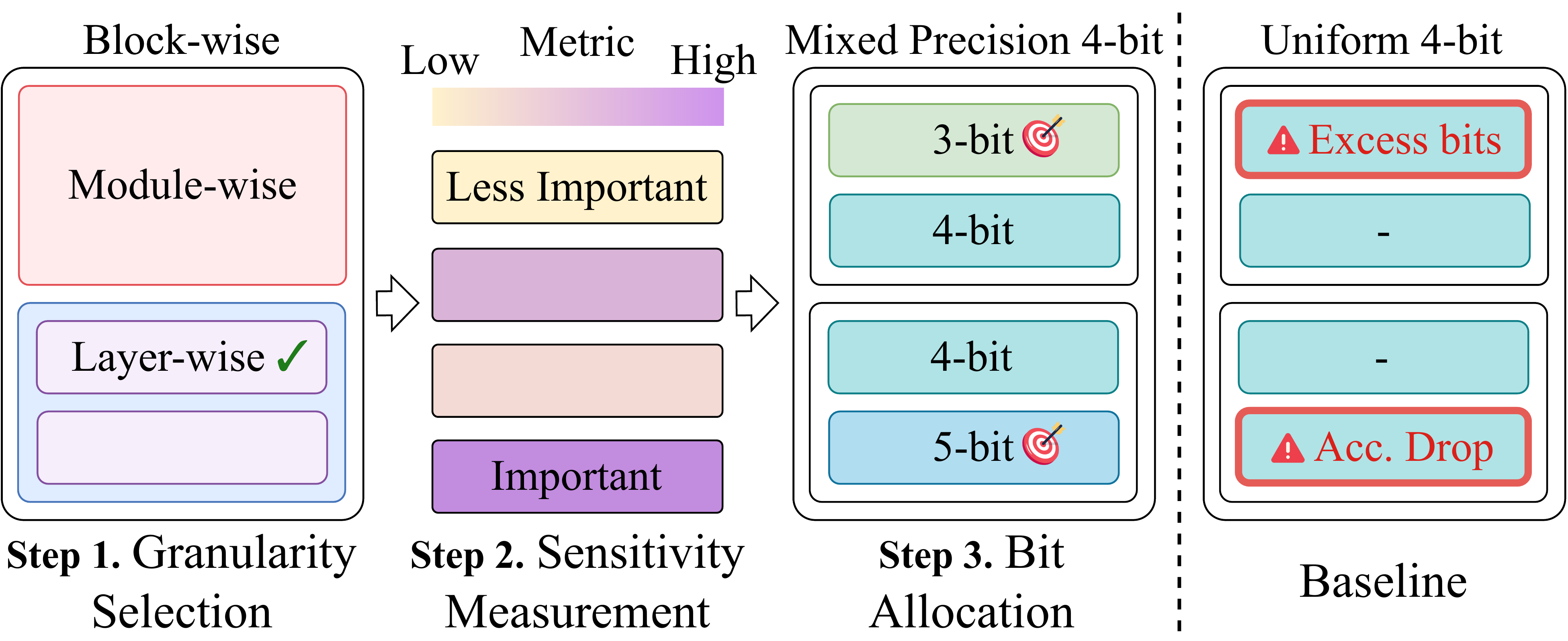}
	\caption{
		Illustration of how metric-based MPQ works.
        They first partition model parameters into groups and measure their sensitivity.
        More bits are allocated to sensitive groups, reducing model size while maintaining performance.
	}
	\label{fig:mpq}
\end{figure}


\subsection{Metric-based Mixed Precision Quantization}
\label{subsec:2_metric_mpq}

In this work, we follow a metric-based approach for Mixed Precision Quantization (MPQ)~\cite{QE, OMPQ}.
As shown in Figure~\ref{fig:mpq}, the three key components of metric-based MPQ are granularity, metric, and bit assignment.
Granularity determines the level (e.g., per-layer, per-channel, or per-tensor) at which bit-width allocation is applied, and is selected in the first step of MPQ to define the partitioning unit.
Sensitivity\footnote{In the remainder of the paper, we use the term `sensitivity' to refer to `quantization sensitivity' for simplicity.}
metric evaluates how each partition responds to quantization and is used to assess the impact of different bit-widths.
Finally, bit assignment determines the bit-width allocation based on sensitivity, prioritizing important segments under a given budget.

 \smallsection{Hessian-based Importance}
 In weight-only quantization of CNN models, researchers introduce the trace $\trace(\mat{H}_i)$ of each layer's Hessian matrix $\mat{H}_i$ as a potential metric, as shown in Lemma~\ref{lemma}~\cite{HAWQ-V2, HAWQ-V3}.
%
\begin{lemma}[Layer importance and Hessian trace]	
	{\color{white}Thank you}
	\label{lemma}
	Assume that for all layers $l_i$ with weight vector $\mat{W}_i$, its gradient $\mat{g}_i=\mat{0}$, Hessian $\hessiani$ is positive semi-definite for target loss $\loss$, and
	$\exists \alpha \in \mathbb{R},  \deltawi=\widehat{\mat{W}}_i-\mat{W}_i=\alpha\sum_{k}{\mat{v}_{ik}}$, where $\widehat{\mat{W}}_i$ is a dequantized form of the quantized weight $\mat{W}'_i$ and $\{\mat{v}_{ik}\}$ are the eigenvectors of
	$\hessiani$.
	For two layers $l_i$ and $l_j$, if $\|\deltawi\|^2_2 = \|\deltawj\|^2_2$ and $\trace(\hessiani) > \trace(\hessianj),$ then,
	\begingroup\small
	\begin{equation*}
		\loss(\widehat{\mat{W}}_i)\geq\loss(\widehat{\mat{W}}_j).
	\end{equation*}
	\endgroup
\end{lemma}
\begin{proof}
	Refer to Lemma 1 of HAWQ-V2~\cite{HAWQ-V2}.
\end{proof}
%

\section{Proposed Method}
\label{sec:method}
\begin{figure}[t]
    \centering
    \includegraphics[width=0.95\linewidth]{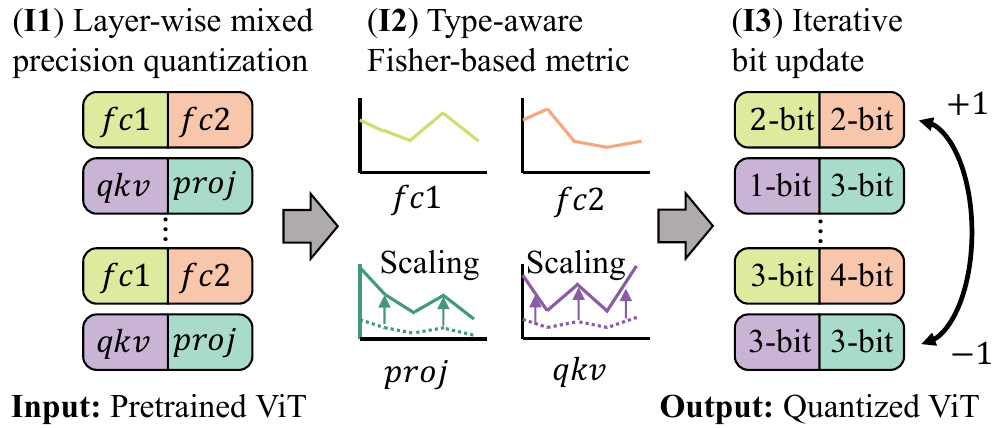}
    \caption{
        Overall architecture of \method.
        Our main ideas are I1) layer-wise mixed-precision quantization, I2) type-aware Fisher-based metric, and I3) iterative bit update.
    }
    \label{fig:method}
\end{figure}


%
We propose \textbf{\method} (\methodfullbold), an accurate mixed-precision PTQ method for ViTs.
These are the three main challenges that must be tackled:

\begin{itemize}[leftmargin=6.7mm, itemsep=0mm, topsep=0mm]
    \item[\textbf{C1.}] \textbf{Coarse-grained granularity.}
    Previous methods enforce uniform quantization within each module, limiting precision flexibility at deeper levels.
    How can we apply finer granularity while ensuring acceleration?

    \item[\textbf{C2.}] \textbf{Mismatch in metric scale across component types.}
    Sensitivity scores vary widely in scale across component types, hindering direct comparison.
    How can we design a metric that enables fair comparison across heterogeneous components?

    \item[\textbf{C3.}] \textbf{Quantization-unaware bit allocation.}
	Bit allocation based on one-shot full-precision metrics is suboptimal,
    as it overlooks how quantizing earlier layers affects the sensitivity of later ones.
	How can we incorporate quantization feedback
    without compromising efficiency?

\end{itemize}

We propose three main ideas to address the challenges:

\begin{itemize}[leftmargin=5.8mm, itemsep=0mm, topsep=0mm]
    \item[\textbf{I1.}] \textbf{Layer-wise mixed precision quantization.}
    We propose an MPQ strategy that flexibly allocates bit-widths at layer level, enhancing both inference speed and performance of the quantized model.

    \item[\textbf{I2.}] \textbf{Type-aware Fisher-based metric.}
    We measure the sensitivity of each layer with the trace of its Fisher information matrix and introduce type-aware scaling to address inter-layer range mismatches.

    \item[\textbf{I3.}] \textbf{Iterative bit update.}
    We assign initial bit-widths via Integer Linear Programming (ILP) and update them iteratively based on estimated error to correct misallocations.

\end{itemize}

Figure~\ref{fig:method} depicts the overall procedure of \method.
%
Given a pre-trained full-precision ViT with $N$ blocks, we structure it into $4N$ layers, with $N$ layers for each layer types: \qkv, \proj, \fcone, and \fctwo.
Next, we quantify the sensitivity of each layer towards quantization by evaluating the trace of its Fisher information matrix, scaled according to the layer type.
With this sensitivity metric, we determine the initial bit-widths by solving an ILP problem.
After quantizing the ViT with the initial bit-widths, \method iteratively updates the bit allocation with the layer-wise reconstruction error, ensuring that the updated bit allocation reflects the quantization-induced error.
%
%
Algorithm of \method is formulated as Alg.~\ref{alg:method} (see Section \textit{`Algorithm'}).
For clarity, we reference the relevant lines for each idea.
Note that \method is compatible with any quantization method for ViTs;
we adopt AdaLog~\shortcite{AdaLog} as the baseline for best performance.

\subsection{Layer-wise Mixed Precision Quantization}
\label{subsec:I1}

\smallsection{Observation}
We present an empirical observation that layer-wise sensitivity varies significantly, making module-based MPQ approaches suboptimal.
We prepare a full-precision DeiT-S~\shortcite{DeiT} model and quantize each component following AdaLog~\cite{AdaLog} while keeping other components unchanged to ignore their effects.
Figure~\ref{fig:granularity} shows that accuracy degradation varies significantly not only across (a) blocks and modules, but also (b) individual layers.
Within each module, \qkv and \proj in \msa, and \fcone and \fctwo in \mlp show notably different sensitivities, especially in the middle layers (16–32).
This variation in sensitivity is consistently observed across different models (see Section \textit{`Further Experiments'}).
However, module-wise MPQ methods~\cite{VT-PTQ} assign the same bits to all components within a module, ignoring this sensitivity variation.

\smallsection{Solution}
Motivated by this observation, we allocate varying bit-widths for each layer to reflect their sensitivity differences.
To achieve this, we propose a layer-wise MPQ scheme that determines its granularity at the layer level.
Finer-grained quantization beyond the module level improve sensitivity estimation, but often introduce hardware inefficiencies.
In particular, granularity finer than the layer level (e.g., per-channel or per-weight) introduces varying bit-widths within the same weight or activation matrix, requiring frequent conversions to full precision during computation, leading to substantial runtime overhead.
By contrast, layer-wise quantization offers a good trade-off between inference speed and performance, as layers are the smallest units compatible with low-bit kernels~\cite{Any-PrecisionLLM}.
%
%
\subsection{Type-aware Fisher-based Metric}
\label{subsec:I2}

\smallsection{Observation}
Designing an accurate and efficient metric to estimate each component’s quantization sensitivity is crucial for metric-based MPQ for ViTs.
VT-PTQ estimates the quantization sensitivity of \msa and \mlp modules by evaluating the nuclear norm of attention and feature maps, respectively.
However, this approach has two major limitations.
First, the metric relies on empirical intuition, which lacks a solid theoretical foundation.
A greater nuclear norm indicates more information in the matrix, which may suggest an important layer, but it does not directly correlate with the performance.
Second, the metric exhibits a large scale gap between two module types, \msa and \mlp.
Figure~\ref{fig:metric}(a) illustrates the metric values measured on a DeiT-S model with VT-PTQ.
The metric values for \msa and \mlp differ by 10 to 40 times, making direct comparison between inter-type modules challenging.
Hence, developing a metric with theoretical justification and scaling consistency is essential.

\begin{figure}[t]
	\centering
	\includegraphics[width=\linewidth]{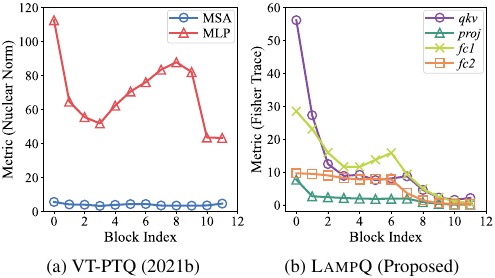}
	\caption{
			\label{fig:metric}
		Comparison of metric values between (a) VT-PTQ~\shortcite{VT-PTQ} and (b) \method for a DeiT-S model.
	}
\end{figure}


\smallsection{Solution}
	In metric‑based MPQ, a sensitivity score $\Omega_i$ is required to quantify the expected degradation from quantizing each layer $l_i$; $\Omega_i$ later guides the bit‑width allocation.
	Our idea is to obtain this score from the trace $\trace(\fisher)$ of the Fisher information matrix with type‑aware scaling $\alpha_t$ for fair comparison across different layer types (lines 1-10 of Alg. 1).

As shown in Lemma~\ref{lemma}, the Hessian trace $\trace(\hessian)$  is an effective sensitivity metric for weight-only MPQ in CNNs~\cite{HAWQ-V2, HAWQ-V3}.
Motivated by this, we adopt the Hessian trace as our initial basis for layer sensitivity of ViTs under a simplified four-linear-layer formulation.
However, directly applying this metric into ViTs faces two challenges.
First, computing $\hessiani$ for each layer $l_i$ requires second-order derivatives with respect to all parameters, leading to significant overhead.
Second, as ViTs contain four layer types while CNNs have only one, assuming that all types exhibit the same accuracy change per unit metric value is invalid.

To address these challenges, we first replace $\trace(\hessiani)$ with $\trace(\fisheri)$,
noting that the expected Hessian $\hessiani$ of each layer approximates its Fisher information matrix $\fisheri$ as in Lemma~\ref{lemma:fisher}.
%
\begin{lemma}[Hessian and Fisher information matrices]
	{\color{white}Th}
	\label{lemma:fisher}
	Assume that loss $\loss$ is the negative log-likelihood.
	Let $\theta$, $p(y|x, \theta)$, $q(x)$ as the model parameters, conditional probability of $y$ given input $x$, and an empirical distribution of $x$, respectively.
	Then, the expected Hessian of layer $l_i$
	%
	\begin{equation*}
		\mathbb{E}_{x \sim q(x)} \mathbb{E}_{y \sim p(y|x,\theta)}[\hessiani] \approx \fisheri,
	\end{equation*}
	where $\fisheri$ denotes the Fisher information matrix of layer $l_i$.
\end{lemma}
\begin{proof}
	Refer to Section \textit{`Theoretical Analysis'}.
\end{proof}
%
Next, we need to normalize sensitivity metrics across layer types, so that their values correspond to an equivalent amount of accuracy degradation.
To achieve this, we introduce a type‑specific scaling factor $\alpha_t$, which converts a unit metric value into an equivalent amount of accuracy drop for each layer type.
We define the type‑scaled sensitivity metric $\Omega_i$ for layer $l_i$ as $\Omega_i = \alpha_{t} \trace(\fisheri),$
where $\alpha_{t} \in \{ \alpha_{\qkv}, \alpha_{\proj}, \alpha_{\fcone}, \alpha_{\fctwo} \}$ is assigned based on the layer type $t$.
This scaling enables fair comparison of sensitivity across heterogeneous components by expressing their metric values in terms of expected accuracy impact.

Then, how can we determine $\alpha_t$ without investigating all layers which leads to an excessive computation?
Our idea is to approximate it from a few sampled layers of each type.
Specifically, we define $\alpha_t = A^{(t)}/\trace(\fisher^{(t)})$ as the ratio between the average accuracy drop $A^{(t)}$ and the average Fisher trace $\trace(\fisher^{(t)})$ when $\mu$ layers of type $t$ are quantized to $\beta$-bit ($\mu \leq L$ and $\beta < 32$).
%
As a result, \method effectively scales metric values across different types, as shown in Figure~\ref{fig:metric}(b).

\begin{figure}[t]
	\centering
	\includegraphics[width=\linewidth]{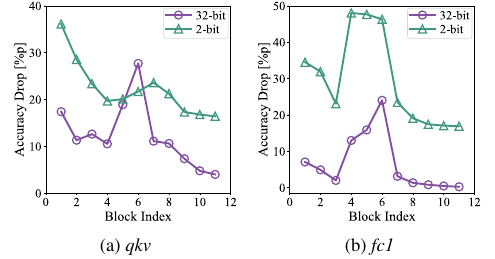}
	\caption{
		\label{fig:iterative}
		%
		Accuracy drop by 1-bit quantization of (a) \qkv or (b) \fcone layers.
		32-bit and 2-bit denote a DeiT-S model and its variant with the first block quantized to 2-bit, respectively.
	}
\end{figure}


\subsection{Iterative Bit Update}
\label{subsec:I3}
\smallsection{Observation}
We present an empirical observation that quantized models show different sensitivity patterns compared to full-precision models.
%
Figure~\ref{fig:iterative} depicts the accuracy drop in a full-precision (32-bit) DeiT-S model and its variant (2-bit), where only the first block is quantized to 2-bit, when quantizing each layer to 1-bit following AdaLog while keeping others unchanged.
The sensitivity trend of the 2-bit variant differs significantly (specifically in blocks 4-8) from its original,
highlighting the importance of accounting for quantization effects.
However, previous approaches evaluate sensitivity only once and assign bit-widths accordingly, thereby failing to reflect this trend difference.

\smallsection{Solution}
To mitigate this challenge, we design a two-step scheme that 1) first assigns bit-widths by solving an Integer Linear Programming (ILP) problem, and then 2) iteratively updates them based on the estimated reconstruction error.
In this way, \method dynamically adapts to evolving sensitivity patterns, leading to more accurate bit configurations.

\smallsection{Step 1. Initial assignment}
Our goal is to quickly initialize bit-widths that minimize quantization error while satisfying the average bit-width constraint.
That is, we want to determine the bit allocation $\{B_i^*\}$ that minimizes total error:
%
\begingroup \small
\begin{equation*}
	\label{eq:ilp}
	\{B_i^*\} = \argmin_{\{B_i\}} \sum_{i=1}^{4N} {\Omega_i^{(B_i)}},
	\textnormal{s.t.}\sum_{i=1}^{4N} {c_i B_i} \leq b_{t} \sum_{i=1}^{4N} {c_{i}},
\end{equation*}
\endgroup
where $\Omega_i^{(B_i)}$ denotes the estimation error at $B_i$-bit and $c_i$ is the number of parameters of layer $l_i$.
The key challenge is to define $\Omega_i^{(B_i)}$ as a proper proxy for quantization error and to optimize bit allocation efficiently under this definition.
We begin by identifying the key requirements for $\Omega_i^{(B_i)}$:
it should (1) reflect the sensitivity of each layer so that more critical layers receive higher bit-widths,
(2) increase the penalty as bit‑width decreases reflecting the amplified error at lower precisions,
and (3) be computationally efficient to evaluate.
To meet these requirements, we extend our type‑aware Fisher‑based sensitivity metric $\Omega_i$ with an additional penalty to model error growth at lower bit‑widths,
defining $\Omega_i^{(B_i)} = \gamma^{-B_i} \Omega_i$ with penalizing factor $\gamma > 1$.
With the defined $\Omega_i^{(B_i)}$, the bit allocation problem becomes an ILP task.
We adopt a fast ILP solver to efficiently find the global optimum under the average bit‑width constraint,
providing a practical initialization for subsequent iterative refinement.
%
%
Compared to existing methods that construct a Pareto frontier, \method accelerates the allocation process over 250$\times$ (see Section \textit{`Further Experiments'}).


\begin{table}[t]
	\centering
	\caption{
		Relative values of estimated $\loss_{recon}^{(B_i)}(l_i)$ when activation $\mat{X}_i$ follows the Gaussian distribution $\mathcal{N}(0, 1)$.
		%
	}
		\setlength{\tabcolsep}{12pt}
		\small
		\begin{tabular}{ccr}
			\hline
			\toprule
			$B_i$ & $\loss_{recon}^{(B_i)}(l_i)$ & $\frac{\loss_{recon}^{(B_i-1)}(l_i)}{  \loss_{recon}^{(B_i)}(l_i)}$
			\\
			\midrule
			1 & $1.000$ E$+0$ & - \\
			2 & $1.209$ E$-2$ & $\times$ 82.74 \\
			3 & $1.835$ E$-3$ & $\times$ 6.59 \\
			4 & $3.903$ E$-4$ & $\times$ 4.70 \\
			5 & $9.093$ E$-5$ & $\times$ 4.29 \\
			6 & $2.199$ E$-5$ & $\times$ 4.13 \\
			7 & $5.339$ E$-6$ & $\times$ 4.12 \\
			8 & $1.324$ E$-6$ & $\times$ 4.03 \\
			\bottomrule
			\hline
		\end{tabular}
	\label{tab:recon}
\end{table}


\def\arraystretch{0.75}
\begin{table*}[t]
	\centering
	\caption{
		Image classification accuracy [\%] of quantized ViTs on the ImageNet dataset.
		W$B$A$B$ denotes that weights and activations are quantized into $B$-bit.
		\MP indicates MPQ.
		Note that \method achieves the highest accuracy in all cases.
	}
		\setlength{\tabcolsep}{9.2pt}
		\small
		\begin{tabular}{lccccccccc}
			\hline
			\toprule
			\multirow{2}[3]{*}{\textbf{Method}} & \multirow{2}[3]{*}{\textbf{W/A}} & \multicolumn{2}{c}{\textbf{ViT}} & \multicolumn{3}{c}{\textbf{DeiT}} & \multicolumn{2}{c}{\textbf{Swin}} & \multirow{2}[3]{*}{\textbf{Average}} \\
			\cmidrule(lr){3-4} \cmidrule(lr){5-7} \cmidrule(lr){8-9}
			& &  \textbf{ViT-S} &\textbf{ViT-B} & \textbf{DeiT-T} & \textbf{DeiT-S} & \textbf{DeiT-B} & \textbf{Swin-S} & \textbf{Swin-B} & \\
			\midrule
			Full-Precision & 32/32 & 81.38 & 84.53 & 72.13 & 79.83 & 81.80 & 83.23 & 85.27 & 81.17\\
			\midrule
			QDrop~\shortcite{QDrop} & 4/4 &
			17.77 & 21.72 & 31.65 & 35.79 & 65.47 & 78.92 & 80.49 & 47.40 \\
			PTQ4ViT~\shortcite{PTQ4ViT} & 4/4 &
			42.57 & 30.69 & 36.96 & 34.08 & 64.39 & 76.09 & 74.02 & 51.26 \\
			PD-Quant~\shortcite{PD-Quant} & 4/4 &
			32.64 & 34.86 & 58.50 & 64.85 & 73.76 & 77.04 & 75.84	 & 59.64 \\
			RepQ-ViT~\shortcite{RepQ-ViT} & 4/4 &
			65.05 & 68.48 & 57.43 & 69.03 & 75.61 & 79.45 & 78.32	 & 70.48 \\
			OPTQ~\shortcite{OPTQ} & 4/4 &
			67.59 & 75.12 & 58.96 & 70.85 & 76.10 & 80.17 & 81.08	 & 72.84 \\
			ERQ~\shortcite{ERQ} & 4/4 &
			68.91 & 76.63 & 60.29 & 72.56 & \underline{78.23} & 80.74 & 82.44 & 74.26 \\
			AdaLog~\shortcite{AdaLog} & 4/4 &
			72.75 & 79.68 & 63.52 & 72.06 & 78.03 & 80.77 & 82.47 & 75.61 \\
			$\text{VT-PTQ}^{\dagger}$~\shortcite{VT-PTQ} & 4\MP/4\MP &
			\underline{73.69} & \underline{80.10} &
			\underline{63.90} & \underline{72.78} &
			\underline{78.30} & \underline{80.96} &
			\underline{82.80} & \underline{76.07} \\
			\textbf{\method (Proposed)} & 4\MP/4\MP &
			\textbf{74.02} & \textbf{81.91} & \textbf{65.71} & \textbf{75.40} & \textbf{79.24} & \textbf{81.76} & \textbf{83.87} & \textbf{77.42} \\
			\midrule
			
			QDrop~\shortcite{QDrop} & 3/3 &
			4.44 & 8.00 & 30.73 & 22.67	& 24.37	& 60.89	& 54.76	& 29.41 \\
			PTQ4ViT~\shortcite{PTQ4ViT} & 3/3 &
			0.01 & 0.01 & 0.04 & 0.01 & 0.27& 0.35& 0.29 & 0.14 \\
			RepQ-ViT~\shortcite{RepQ-ViT} & 3/3 &
			0.43 & 0.14 & 0.97 & 4.37 & 4.84 & 8.84 & 1.34 & 2.99 \\
			AdaLog~\shortcite{AdaLog} & 3/3 &	
			13.88 & 37.91 & 31.56 & 24.47 & 57.45 & 64.41 & 69.75 & 42.78 \\
			$\text{VT-PTQ}^{\dagger}$~\shortcite{VT-PTQ} & 3\MP/3\MP &
			\underline{16.62} & \underline{42.13} & \underline{32.98} & \underline{26.37} & \underline{60.14} & \underline{69.80} & \underline{73.51} & \underline{45.94} \\
			\textbf{\method (Proposed)} & 3\MP/3\MP &
			\textbf{23.06} & \textbf{48.53} & \textbf{37.54} & \textbf{45.38} & \textbf{61.44} & \textbf{70.91} & \textbf{75.82} & \textbf{51.81} \\
			
			\bottomrule
			\hline
			\addlinespace[0.5ex]
			\multicolumn{10}{l}{$\dagger$: AdaLog quantization with bit allocation by VT-PTQ.}
		\end{tabular}
	\label{tab:q1}
\end{table*}


\smallsection{Step 2. Error-based iterative updates}
Our goal is to further update bit assignments to reduce quantization error,
while keeping the same average bit-widths assigned from the initial assignment step.
For the goal, we iteratively update bit assignments;
at each iteration, we select two layers, and increase the first layer's bit-width by 1, while decreasing the second layer's bit-width by 1.
In this way, we guarantee that the average bit assignment remains the same.

Then, how can we find the best two layers to increase/decrease their bit-widths?
Our idea is to make decision based on gain and degradation which we define for this purpose.
The gain of a layer $l_i$ is defined to be the reduction in error when increasing the bit-width of $l_i$ by 1,
while the degradation of it is the increase in error when decreasing the bit-width of $l_i$ by 1.
Specifically, the gain and the degradation of $l_i$ are given by
$\loss_{recon}^{(B_i - 1)}(l_i) - \loss_{recon}^{(B_i)}(l_i)$ and $\loss_{recon}^{(B_i)}(l_i) - \loss_{recon}^{(B_i + 1)}(l_i)$, respectively,
where
$\loss_{recon}^{(B_i)}(l_i)$ is the layer-wise reconstruction error of $l_i$ at $B_i$-bit:
%
\begingroup \small
\begin{equation*}
	\label{eq:recon}
	\loss_{recon}^{(B_i)}(l_i) =
	\|\weightiqdq^{(B_i)}\actiqdq^{(B_i)}-\mat{W}_i\mat{X}_i\|^2_F / \|\mat{W}_i\mat{X}_i\|^2_F,
\end{equation*}
\endgroup
where $\mat{W}_i$ and $\mat{X}_i$ are weights and activations of layer $l_i$, respectively,
and
$\weightiqdq^{(B_i)}$ and $\actiqdq^{(B_i)}$ are dequantized weights and activations after quantization, respectively.
This updates are performed iteratively until the model accuracy on the sample dataset $\mathbb{D}$ stops improving (lines 22-24 of Alg. 1).

\smallsection{Efficient estimation of error gain and degradation}
Na\"ively evaluating gain and degradation requires measuring loss for all bit‑width reallocation pairs, which is prohibitively expensive for ViTs.
To make this tractable, we approximate gain and degradation using statistical estimation, eliminating the need to recompute reconstruction errors for every bit‑width change.
Concretely, we approximate the ratio $\loss_{recon}^{(B_i - 1)}(l_i)/\loss_{recon}^{(B_i)}(l_i)$ with its expectation $\mathbb{E}(\cdot)$ over weight and activation distributions, as formalized in Lemma~\ref{lemma:ratio}.
For example, the gain $\loss_{recon}^{(B_i - 1)}(l_i) - \loss_{recon}^{(B_i)}(l_i)$ can be rewritten as $\loss_{recon}^{(B_i)}(l_i)\big(\frac{\mathbb{E}(\loss_{recon}^{(B_i - 1)}(l_i))}{\mathbb{E}(\loss_{recon}^{(B_i)}(l_i))} - 1\big)$, where the ratio term is substituted with the pre‑computed expectation.
%
\begin{lemma}[Expected ratio of reconstruction losses]
	{\color{white}Thank}
	\label{lemma:ratio}
	Assume that given a layer $l_i$, its weight $\weightifp$ and activation $\actifp$ are independent.
	Then, the expected ratio between the reconstruction error at $(B_i-1)$-bit and $B_i$-bit is:
	%
	\begingroup\small
	\[
		\frac{\mathbb{E}(\loss_{recon}^{(B_i - 1)}(l_i))}
		{\mathbb{E}(\loss_{recon}^{(B_i)}(l_i))}
		= \frac{k(B_i-1; l_i)}{k(B_i; l_i)},
	\]
	\endgroup
	where $k(\cdot)=\mathbb{E}\Big(\big(\deltaw^{(B_i)} X + W \deltaw^{(B_i)} + \deltaw^{(B_i)} \deltax^{(B_i)}\big)^2\Big)$,
	and $\mathbb{E}(\cdot)$ denotes the expectation over random variables
	$W,X$ (distributed as elements of $~\weightifp,\actifp$)
	and $\deltaw^{(B_i)},\deltax^{(B_i)}$ (distributed as elements of $~\weightiqdq^{(B_i)}-\weightifp,\actiqdq^{(B_i)}-\actifp$).
\end{lemma}
\begin{proof}
	Refer to Section \textit{`Theoretical Analysis'}.
\end{proof}
%
We compute the expectations in Lemma~\ref{lemma:ratio} by modeling weights and activations as Gaussian $\mathcal{N}(0,1)$, following prior observations~\cite{PTQ4ViT, RepQ-ViT}.
Pre‑computed expectations for Gaussian distribution are summarized in Table~\ref{tab:recon}.
We detail the derivation process in Section \textit{`Theoretical Analysis'} and Table~\ref{tab:exp_norm}.

\section{Experiments}
\label{sec:experiments}
We present experimental results to answer the questions:

\begin{itemize}[leftmargin=7mm, itemsep=0mm, topsep=0mm]
    \item[\textbf{Q1.}] \textbf{Image classification accuracy.}
        How accurate are the quantized ViTs with \method in image classification?

    \item[\textbf{Q2.}] \textbf{Object detection precision.}
        How accurate are the quantized ViTs by \method in object detection?

    \item[\textbf{Q3.}] \textbf{Application to zero-shot quantization.}
        How accurate are quantized ViTs in zero-shot quantization?

    \item[\textbf{Q4.}] \textbf{Ablation study.}
        Does each component of \method help improve performance?
%
    \item[\textbf{Q5.}] \textbf{Case study on bit allocation.}
        Does \method assign bits according to quantization sensitivity?
\end{itemize}

\def\arraystretch{0.8}
\begin{table*}[t]
	\centering
	\caption{
		Precision of quantized models on MS-COCO dataset.
		Note that \method achieves the best performance in all cases.
	}
		\setlength{\tabcolsep}{8.4pt}
        \small
		\begin{tabular}{lcccccccccc}
			\hline
			\toprule
			& & \multicolumn{4}{c}{\textbf{Mask R-CNN}} & \multicolumn{4}{c}{\textbf{Cascade Mask R-CNN}} &
			\\
			\cmidrule(lr){3-6} \cmidrule(lr){7-10}			\textbf{Method} & \textbf{W/A} & \multicolumn{2}{c}{\textbf{Swin-T}} & \multicolumn{2}{c}{\textbf{Swin-S}} & \multicolumn{2}{c}{\textbf{Swin-T}} & \multicolumn{2}{c}{\textbf{Swin-S}} & \textbf{Average}\\
			\cmidrule(lr){3-4} \cmidrule(lr){5-6} \cmidrule(lr){7-8} \cmidrule(lr){9-10}
			&  & $\text{AP}^\text{box}$ & $\text{AP}^\text{mask}$ &  $\text{AP}^\text{box}$ & $\text{AP}^\text{mask}$ &  $\text{AP}^\text{box}$ & $\text{AP}^\text{mask}$ &  $\text{AP}^\text{box}$ & $\text{AP}^\text{mask}$ & \\
			\midrule
			Full-Precision & 32/32 & 46.0 & 41.6 & 48.5 & 43.3 & 50.4 & 43.7 & 51.9 & 45.0 & 46.3\\
			\midrule
			QDrop~\shortcite{QDrop} & 4/4 & 12.4 & 12.9 &	42.7 & 40.2 & 23.9 & 21.2 & 24.1 & 21.4 & 24.9 \\
			PD-Quant~\shortcite{PD-Quant} & 4/4 & 17.7 & 18.1 & 32.2 & 30.9 & 35.5 & 31.0 & 41.6 & 36.3 & 30.4 \\
			RepQ-ViT~\shortcite{RepQ-ViT} & 4/4 & 36.1 & 36.0 & 44.2 & 40.2 & 47.0 & 41.4 & 49.3 & 43.1 & 42.2\\
			OPTQ~\shortcite{OPTQ} & 4/4 & 36.3 & 36.3& 42.9 & 40.2 & 47.1 & 41.5 & 49.2 & 43.2 & 42.1\\
			ERQ~\shortcite{ERQ} & 4/4 & 36.8 & 36.6 & 43.4& 40.7 & 47.9 & 42.1 & 50.0 & 43.6 & 42.6\\
			AdaLog~\shortcite{AdaLog} & 4/4 &
			39.1 & \underline{37.7} & \underline{44.3} & 41.2 & 48.2 & 42.3 & 50.6 & 44.0 & 43.4 \\
			$\text{VT-PTQ}^{\dagger}$~\shortcite{VT-PTQ} & 4\MP/4\MP &
			\underline{39.2} & \underline{37.7} &
			\underline{44.3} & \underline{41.3} &
			\underline{48.3} & \underline{42.5} &
			\underline{50.9} & \underline{44.2} &
			\underline{43.6} \\
			\textbf{\method (Proposed)} & 4\MP/4\MP & \textbf{39.8} & \textbf{38.4} & \textbf{44.9} & \textbf{41.8} & \textbf{49.0} & \textbf{43.1} & \textbf{51.1} & \textbf{44.5} & \textbf{44.1}  \\
			\bottomrule
			\hline
			\addlinespace[0.5ex]
			\multicolumn{10}{l}{$\dagger$: AdaLog quantization with bit allocation by VT-PTQ.}
		\end{tabular}
	\label{tab:q2}
\end{table*}


\subsection{Experimental Setup}
\label{subsec:setup}
	We briefly introduce the experimental setup.
	Refer to Section \textit{`Experimental Setup'} for further details.

\smallsection{Setup}
We evaluate our method with ViT~\shortcite{ViT}, DeiT~\shortcite{DeiT}, and Swin~\shortcite{Swin} models on ImageNet~\shortcite{ImageNet} dataset for image classification and zero-shot quantization, and MS COCO~\shortcite{MS_COCO} dataset for object detection.

\smallsection{Competitors and Details}
We compare \method with nine existing PTQ methods for image classification and object detection tasks.
For zero-shot quantization, we set PSAQ-ViT~\shortcite{PSAQ-ViT} as our baseline.
%
We follow the settings from PSAQ-ViT~\shortcite{PSAQ-ViT} and AdaLog~\shortcite{AdaLog} for fair comparison.

\subsection{Image Classification Accuracy (Q1)}
\label{subsec:q1}
We evaluate the image classification accuracy of ViTs quantized by \method against existing quantization methods in Table~\ref{tab:q1}.
%
%
\method consistently improves quantized models across diverse bit-widths and architectures, achieving up to 5.87\%p higher average accuracy.
Notably, \method becomes increasingly effective as the bit‑width decreases,
highlighting its robustness in challenging low‑precision settings.

\begin{table}[t]
    \def\arraystretch{0.6}
	\centering
	\caption{
		Zero-shot quantization accuracy [\%] on ImageNet dataset.
		\method consistently shows the best performance.
	}
    \small
    \setlength{\tabcolsep}{5pt}
        \begin{tabular}{lccccc}
			\hline
			\toprule
			\multirow{2}[3]{*}{\textbf{Method}} & \multirow{2}[3]{*}{\textbf{W/A}} & \multicolumn{2}{c}{\textbf{DeiT}} & \multicolumn{2}{c}{\textbf{Swin}} \\
			\cmidrule(lr){3-4} \cmidrule(lr){5-6}
			& & \textbf{DeiT-T} & \textbf{DeiT-S} & \textbf{Swin-T} & \textbf{Swin-S} \\
			\midrule
			Original & 32/32 & 72.21 & 79.85 & 81.35 & 83.20 \\
			\midrule
			PSAQ-ViT & 4/8 & 65.57 & 72.04 & 69.78 & 75.03 \\
			VT-PTQ$^\dagger$ & 4\MP/8\MP
			& \underline{65.65}
			& \underline{72.18}
			& \underline{69.91}
			& \underline{75.09} \\
			\textbf{\method} & 4\MP/8\MP
			& \textbf{66.27}
			& \textbf{72.71}
			& \textbf{70.24}
			& \textbf{75.49} \\
			\midrule
			PSAQ-ViT & 8/8 & 71.56 & 75.97 & 73.54 & 76.68 \\
			VT-PTQ$^\dagger$ & 8\MP/8\MP
			& \underline{71.58}
			& \underline{76.02}
			& \underline{73.63}
			& \underline{76.72} \\
			\textbf{\method} & 8\MP/8\MP
			& \textbf{71.77}
			& \textbf{76.20}
			& \textbf{73.76}
			& \textbf{76.85} \\
			\bottomrule
			\hline
			\addlinespace[0.5ex]
			\multicolumn{6}{l}{$\dagger$: PSAQ-ViT quantization with bit allocation by VT-PTQ.}
		\end{tabular}
	\label{tab:q3}
\end{table}


\subsection{Object Detection Precision (Q2)}
\label{subsec:q2}
We investigate the effectiveness of \method for quantized models on object detection and instance segmentation.
Table~\ref{tab:q2} shows the average precision of each quantized model's bounding box ($\text{AP}^\text{box}$) and segmentation mask ($\text{AP}^\text{mask}$).
\method achieves the state-of-the-art performance across all settings, validating its generalization to detection tasks.

\subsection{Application to Zero-shot Quantization (Q3)}
\label{subsec:q3}
We investigate the effectiveness of \method in settings without any real data.
Table~\ref{tab:q3} shows the zero-shot quantization accuracy of four ViT models.
\method enhances the quantization accuracy across various models, achieving up to 0.62\%p increase when applied to PSAQ-ViT~\cite{PSAQ-ViT}.
The results show that \method is robust towards dataset quality, effective both for real and synthetic datasets.


\begin{table}[t]
	\def\arraystretch{0.75}
	\centering
	\caption{
		Ablation study of our proposed ideas in \method.
		All ideas of \method effectively enhance the performance.
	}
		\setlength{\tabcolsep}{5pt}
        \small
		\begin{tabular}{lccccc}
			\hline
			\toprule
			\textbf{Method} & \textbf{MPQ} & \textbf{I1} & \textbf{I2} & \textbf{I3} & \textbf{Accuracy}\\
			\midrule
			Base: AdaLog~\shortcite{AdaLog} & \rxmark & \rxmark & \rxmark & \rxmark & 24.47 \\
			Base + VT-PTQ~\shortcite{VT-PTQ} & \gcmark & \rxmark & \rxmark & \rxmark & 26.37 \\
			\midrule
			Base + I1 + I2 & \gcmark & \gcmark & \gcmark & \rxmark & 44.43 \\
			Base + I1 + I3 & \gcmark & \gcmark & \rxmark & \gcmark & 27.87 \\
			\midrule
			\textbf{\method (Proposed)} & \gcmark & \gcmark & \gcmark & \gcmark & \textbf{45.38} \\
			\bottomrule
			\hline
		\end{tabular}
	\label{tab:ablation}
\end{table}


\subsection{Ablation Study (Q4)}
\label{subsec:q4}
We perform an ablation study to show that each main idea of \method improves the performance.
Table~\ref{tab:ablation} summarizes the 3-bit results of a DeiT-S model pre-trained on ImageNet.
%
Our analysis shows that all ideas of \method contribute to the improved performance, with type-aware Fisher-based metric (I2) having the strongest impact of 19.96\%p.
%

\begin{figure}[t]
	\includegraphics[width=\linewidth]{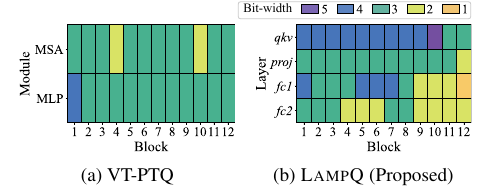}
	\caption{
		\label{fig:case}
		Bit allocation of 3-bit quantized DeiT-S models.
	}
\end{figure}


\subsection{Case Study on Bit Allocation (Q5)}
\label{subsec:q5}
We conduct a case study to examine how \method allocates bits and whether it reflects each layer's sensitivity.
Figure~\ref{fig:case} compares the bit allocation of a 3-bit DeiT-S model by (a) VT-PTQ and (b) \method.
\method captures layer-level sensitivity variations within modules that VT-PTQ misses.
Fine-grained allocation shows \qkv and \fcone requiring higher bits than \proj and \fctwo, consistent to our analysis in Figure~\ref{fig:granularity}(b).

\section{Related Work}
\label{sec:related}

MPQ~\cite{MPSurvey} quantizes different components of a model with varying bit-precisions.
There are four main approaches: learning, Reinforcement Learning (RL), Neural Architecture Search (NAS), and metric-based solutions.
Learning-based methods~\cite{SDQ, NIPQ} treat bit-widths as trainable parameters and update them based on loss gradients.
RL-based methods~\cite{AutoQ, MetaMix} leverage a RL agent to determine the allocation policy.
NAS-based solutions~\cite{BP-NAS, JAQ} explore the bit selection space through an automated search process.
Metric-based methods rely on statistical properties such as quantization entropy~\cite{QE} and orthogonality~\cite{OMPQ}.
Among them, \method quantifies sensitivity based on the trace of Fisher information matrix.
In this way, \method more precisely allocates bit-widths to layers and achieves better performance. 

\section{Conclusion}
\label{sec:conclusion}
We propose \textbf{\method}, an accurate Mixed Precision Quantization (MPQ) method for Vision Transformers, which is adaptable to any existing quantization method.
\method outperforms existing methods in diverse tasks such as image classification, object detection, and zero-shot quantization.
Future works include extending \method into various settings such as vision language models and generation tasks. 


\section{Acknowledgments}
This work was supported by Mobile eXperience (MX) Business, Samsung Electronics Co., Ltd.
U Kang is the corresponding author.

\bibliography{main}

\clearpage
\appendix

\section{Appendix}
We provide additional information on \method including frequently used notations, detailed experimental setup, algorithm, theoretical analysis, and further experimental results.

\section{Notation}
\label{appendix:notation}
We summarize the notations in the paper as Table \ref{tab:notation}.

\begin{table}[ht]
	\centering
	\small
	\setlength{\tabcolsep}{4pt}
    \def\arraystretch{1.1}
	\caption{Notation description.}
		\begin{tabular}{%
			>{\centering\arraybackslash}p{0.24\columnwidth}%
			>{\raggedright\arraybackslash}p{0.72\columnwidth}%
		}
			\hline
			\toprule
			\textbf{Symbol} & \textbf{Description} \\
			\midrule

			$\fpmodel$
			& Pre-trained model with parameters $\theta$ \\

			$\quantmodel$
			& Quantized model with parameters $\quantparam$ \\

			\midrule
			
			$\imageI \in \realset^{H \times W \times C}$
			& Input image\\

			$H$
			& Height of the image $\imageI$ \\

			$W$
			& Width of the image $\imageI$ \\
			
			$C$
			& Channels of the image $\imageI$ \\

			$P$
			& Patch size \\

			$L=HW/P^2$
			& Number of patches \\
			
			$N$
			& Number of ViT blocks \\

			\midrule

			$s$
			& Scaling factor in quantization\\
			
			$z$
			& Integer offset in quantization\\
			
			$(r_{min}, r_{max})$
			& Lower and upper bounds of a given matrix $\matrixfp$ \\
			
			$\matrixquant$
			& Quantized matrix of $\matrixfp$ \\
			
			$\matrixqdq$
			& Dequantized matrix of $\matrixquant$\\
			
			\midrule

			$\mat{W}_i$
			& Weight matrix of layer $l_i$ in the model \\

			$\mat{X}_i$
			& Activation of layer $l_i$ in the model \\
			
			$\deltawi, \deltaxi$
			& Perturbations of $\mat{W}_i$ and $\mat{X}_i$ after quantization \\
			
			$B_i$
			& Bit-width assigned to layer $l_i$ \\
			
			$\mat{H}_i$
			& Hessian matrix of layer $l_i$ \\
			
			$\mat{F}_i$
			& Fisher information matrix of layer $l_i$ \\
			
			$\trace(\cdot) $
			& Trace operator \\
			
			$\mathbb{E}(\cdot)$ & Expectation function\\

			\midrule
			
			$\Omega_i$
			& Proxy of the loss $\loss(\weightiqdq)$ of layer $l_i$ \\

			$\alpha_{type}$
			& Scaling factor for each type of layer \\
			
		
			$\mu$
			& Number of layers for $\alpha_{type}$ calculation\\
			
            $\beta$
			& Bit-width for $\alpha_{type}$ calculation \\

			$\mathbb{B}$
			& Pre-defined set of possible bit selections \\

			$\gamma$
			& Hyperparameter controlling the degree of MPQ \\

			\bottomrule
			\hline
		\end{tabular}
	\label{tab:notation}
\end{table}


\section{Experimental Setup}
\label{appendix:exp_setting}
\def\arraystretch{1.2}
\begin{table}[t]
	\centering
	\caption{Hyperparameter ranges for \method.}
		\setlength{\tabcolsep}{12pt}
		\begin{tabular}{cl}
			\hline
			\toprule
			\textbf{Hyperparameter} & \textbf{Range} \\
			\midrule
			$\mu$ & [$\frac{1}{4} N$, $\frac{1}{2} N$, $\frac{3}{4} N$, $N$]\\
			$\mathbb{B}$ & \makecell{[
				\{1, 2, 3, 4, 5\},
				\{1, 2, 3, 4\}, \\
				\{2, 3, 4, 5\},
				\{2, 3, 4\}
				]}
			\\
			$\gamma$ & [2, 4, 6, 8, 10, 15, 20] \\
			$\beta$ & [1, 2, 3, 4] \\
			\bottomrule
			\hline
		\end{tabular}
	\label{tab:hyperparam_range}
\end{table}


We describe the experimental setup, including datasets, competitors, hyperparameters, and implementation details.

\subsection{Datasets}
We utilize ImageNet (ILSVRC 2012)~\cite{ImageNet} dataset to evaluate the classification accuracy of the quantized model obtained by \method.
We evaluate quantized models on MS-COCO~\cite{MS_COCO} dataset to check their performance on object detection.
For zero-shot quantization task, we utilize real datasets (e.g. ImageNet) only for evaluation purposes.

\subsection{Competitors}
We summarize the competitors of \method as follows:

\begin{itemize}[leftmargin=4mm, itemsep=0mm]
    \item \textbf{PTQ4ViT}~\cite{PTQ4ViT} observes the power-law distribution of post-Softmax activations and introduce twin uniform quantization to reduce the quantization error in such activation distributions.
 	\item \textbf{APQ-ViT}~\cite{APQ-ViT} introduces bottom-elimination blockwise calibration to apply quantization in a blockwise manner, and Matthew-effect preserving quantization to maintain the power-law distribution of post-Softmax activations.
 	\item \textbf{QDrop}~\cite{QDrop} randomly drops the quantization of activation for smoother loss landscapes in test samples.
 	\item \textbf{PD-Quant}~\cite{PD-Quant} mitigates the overfitting of quantized models toward calibration samples by considering global information from the prediction difference between quantized and full precision models.
 	\item \textbf{RepQ-ViT}~\cite{RepQ-ViT} decouples the quantization and inference processes, by initially applying $\log\sqrt{2}$ quantization and reparameterizing the scales to hardware-friendly $\log2$ quantization in inference stage.
 	\item \textbf{OPTQ}~\cite{OPTQ} is an approximate second-order method for quantizing Transformer-based models, by utilizing second-order information.
 	\item \textbf{ERQ}~\cite{ERQ} reduces activation and weight quantization error by formulating the quantization process as a ridge regression problem.
 	\item \textbf{AdaLog}~\cite{AdaLog} is a non-uniform quantizer that optimizes the logarithmic base to accommodate the power-law distribution of activations.
 \end{itemize}

 Additionally, we compare \method with PSAQ-ViT~\cite{PSAQ-ViT} for the zero-shot quantization experiment.

\subsection{Hyperparameters}
We conduct a grid search to validate hyperparameters, and select the set with the best performance.
Table \ref{tab:hyperparam_range} reports the searched hyperparameter ranges of \method.
For competitors, we search within the range described in each paper.

\subsection{Implementation Details and Machine}
We follow the settings from PSAQ-ViT~\cite{PSAQ-ViT}, RepQ-ViT~\cite{RepQ-ViT}, and AdaLog~\cite{AdaLog} for fair comparison.
The size $S$ of sample dataset $\mathbb{D}$ is 1,024 with batch size of 32.
The average bit-width of each model is computed as the mean across all parameters, with each layer assigned a uniform bit-width.
We implement \method with PyTorch~\cite{PyTorch} and timm~\cite{Timm} libraries in Python.
We exploit the default solver provided by the PuLP~\cite{PuLP} library, which leverages CBC (COIN-OR Branch-and-Cut)~\cite{CBC} algorithm to solve the ILP problem.
The ILP optimization completes within a few seconds on a standard CPU, thus we do not observe runtime bottlenecks.
For the other methods, we reproduce the result using their open-source code if possible and implement them otherwise.
All of our experiments are done at a workstation with Intel Xeon Silver 4310 and RTX 4090. 

\section{Theoretical Analysis}
\label{appendix:proofs}


\begin{table}[t]
	\centering
	\caption{
		Expected values for Gaussian distribution $\mathcal{N}(0, 1)$.
	}
    \def\arraystretch{1.1}
    \setlength{\tabcolsep}{12pt}
	\begin{tabular}{cccc}
		\hline
		\toprule
		$B_i$ & $\mathbb{E}\big(X^2\big)$ & $\mathbb{E}\big(X\deltax\big)$ & $\mathbb{E}\big((\deltax)^2\big)$\\
		\midrule
		1 & 1 & $1.396$ E$+0$ & $5.212$ E$+0$ \\
		2 & 1 & $1.655$ E$-2$ & $3.359$ E$-1$ \\
		3 & 1 & $7.123$ E$-4$ & $6.109$ E$-2$ \\
		4 & 1 & $1.723$ E$-4$ & $1.330$ E$-2$ \\
		5 & 1 & $4.123$ E$-5$ & $3.113$ E$-3$ \\
		6 & 1 & $1.003$ E$-5$ & $7.538$ E$-4$ \\
		7 & 1 & $2.472$ E$-6$ & $1.855$ E$-4$ \\
		8 & 1 & $6.134$ E$-7$ & $4.601$ E$-5$ \\
		\bottomrule
		\hline
	\end{tabular}
	\label{tab:exp_norm}
\end{table}


\subsection{Proof of Lemma 2}
\label{supp:subsec:proof_lemma2}

\begin{proof}
	We consider a statistical model defined by the conditional probability distribution $p(y \mid x, \theta)$, which gives the probability of observing the target $y$ given an input $x$ and is parameterized by the model parameters $\theta$.
	When the loss function is the negative log-likelihood, the expected Hessian of layer $l_i$ with respect to $\theta$ can be expressed as:
	
	\begin{align*}
		& \mathbb{E}_{x \sim q(x)}\mathbb{E}_{y \sim p(y|x,\theta)}[\hessiani] \\
		&= \mathbb{E}_{x \sim q(x)}\!\Bigl[
		\mathbb{E}_{y \sim p(y|x,\theta)}\!\Bigl[
		-\nabla_\theta^2 \log p\bigl(y \mid x, \theta\bigr)
		\Bigr]\Bigr]\\
		&= \mathbb{E}_{x \sim q(x)}\!\Bigl[
		\mathbb{E}_{y \sim p(y|x,\theta)}\!\Bigl[
		-\frac{\nabla_\theta^2 p(y \mid x, \theta)}{p(y \mid x, \theta)} \\
		&+ \nabla_\theta \log p(y \mid x, \theta)\nabla_\theta \log p(y \mid x, \theta)^\top
		\Bigr]\Bigr]\\
		&\approx \mathbb{E}_{x \sim q(x)}\!\Bigl[
		\mathbb{E}_{y \sim p(y|x,\theta)}\!\Bigl[
		-\frac{\nabla_\theta^2 p(y \mid x, \theta)}{p(y \mid x, \theta)}		\Bigr]\Bigr]+\fisheri\\
		&=\mathbb{E}_{x \sim q(x)}\!\Bigl[
		\sum_{y \in \mathcal{C}} -\frac{\nabla_\theta^2 p(y \mid x, \theta)}{p(y \mid x, \theta)}p(y \mid x, \theta)\Bigr]+\fisheri\\
		&=\mathbb{E}_{x \sim q(x)}\!\Bigl[
		-\nabla_\theta^2 \sum_{y \in \mathcal{C}}p(y \mid x, \theta)\Bigr]+\fisheri\\
		&=\fisheri,
	\end{align*}
	
	where $q(x)$ and $\mathcal{C}$ denote an empirical prior distribution regarding the dataset and the class set, respectively.
	The second equality follows from the derivative identity:
	
	\begin{align*}
		-\nabla_\theta^2 \log p\bigl(y \mid x, \theta\bigr)
		&=-\nabla_\theta \Bigl[\frac{\nabla_\theta p\bigl(y \mid x, \theta\bigr)}{p\bigl(y \mid x, \theta\bigr)} \Bigr]\\
		&=-\frac{\nabla_\theta^2 p(y \mid x, \theta)}{p(y \mid x, \theta)} \\
		& + \nabla_\theta \log p(y \mid x, \theta)\nabla_\theta \log p(y \mid x, \theta)^\top.
	\end{align*}
	
	The approximation step assumes that the empirical prior distribution $q(x)$ is sufficiently close to the prior distribution.
\end{proof}

\subsection{Proof of Lemma 3}
\begin{proof}
We first calculate the expected value of the reconstruction error $\loss_{recon}^{(B_i)}(l_i)$ for the $i$-th layer $l_i$ with bit-width $B_i$.
For weights $\mat{W}_i\in \mathbb{R}^{D'\times D}$ and activations $\mat{X}_i\in\mathbb{R}^{D\times L}$ of the given layer $l_i$, let $\deltawi^{(B_i)}=\weightiqdq^{(B_i)}-\weightifp$ and $\deltaxi^{(B_i)}=\actiqdq^{(B_i)}-\actifp$ be the quantization errors of weights and activations, respectively.
Then, the expected reconstruction error is obtained as follows:

	\begin{align*}
	&&  \mathbb{E}\Big(\loss_{recon}^{(B_i)}(l_i)\Big) \\
	&&= \frac{1}{\|\mat{W}_i\mat{X}_i\|_F^2} \mathbb{E}&\big(\|(\mat{W}_i+\deltawi^{(B_i)})(\mat{X}_i+\deltaxi^{(B_i)})\\
	&&&-\mat{W}_i\mat{X}_i\|_F^2\big) \\
	&&= \frac{1}{\|\mat{W}_i\mat{X}_i\|_F^2}\mathbb{E}&\big(\|\deltawi^{(B_i)}\mat{X}_i + \mat{W}_i\deltaxi^{(B_i)} \\
	&&& + \deltawi^{(B_i)}\deltaxi^{(B_i)}\|_F^2\big) \\
	&&= \frac{1}{\|\mat{W}_i\mat{X}_i\|_F^2}\mathbb{E}&\Big(\sum_{r=1}^{D'}\sum_{c=1}^{L}\big(
		(\deltawi^{(B_i)}\mat{X}_i)_{rc}
		+(\mat{W}_i\deltaxi^{(B_i)})_{rc}\\
	&&& +(\deltawi^{(B_i)}\deltaxi^{(B_i)})_{rc}
		\big)^2\Big) \\
	&&= \frac{1}{\|\mat{W}_i\mat{X}_i\|_F^2}\mathbb{E}&\Big(\sum_{r=1}^{D'}\sum_{c=1}^{L}\sum_{k=1}^{D}\big(
		(\deltawi^{(B_i)})_{rk}(\mat{X}_i)_{kc}\\
	&&&	+(\mat{W}_i)_{rk}(\deltaxi^{(B_i)})_{kc}\\
	&&& +(\deltawi^{(B_i)})_{rk}(\deltaxi^{(B_i)})_{kc}
		\big)^2\Big).
	\end{align*}
	
	We further simplify the summation by the assumption that the weights and activations are independent, which allows us to separate the expectations of the individual terms.
	Remind that random variables $W,X,\deltaw^{(B_i)}$, and $\deltax^{(B_i)}$ are distributed as elements of $\weightifp, \actifp, \deltawi^{(B_i)}$, and $\deltaxi^{(B_i)}$, respectively.
	Then, the expected value of the reconstruction error is expressed as follows:

	\begin{align*}
		& \mathbb{E}\Big(\loss_{recon}^{(B_i)}(l_i)\Big)\\
		=\quad&\frac{D'LD}{ \|\mat{W}_i\mat{X}_i\|_F^2} \mathbb{E}\Big(\big(
			\deltaw^{(B_i)} X
			+ W \deltax^{(B_i)}
			+ \deltaw^{(B_i)} \deltax^{(B_i)}
			\big)^2\Big).
	\end{align*}

	Note that the term $\frac{D'LD}{ \|\mat{W}_i\mat{X}_i\|_F^2}$ is a constant factor that is independent of the quantization bit-width $B_i$.
	Hence, the ratio between the expected reconstruction errors for different bit-widths $B_i$ and $B_i-1$ is obtained as follows:

	\begin{align}
		&\frac{\mathbb{E}(\loss_{recon}^{(B_i - 1)}(l_i))}
		{\mathbb{E}(\loss_{recon}^{(B_i)}(l_i))} \nonumber\\
		=& \frac{\mathbb{E}\Big(\big(
			\deltaw^{(B_i-1)} X
			+ W \deltax^{(B_i-1)}
			+ \deltaw^{(B_i-1)} \deltax^{(B_i-1)}
			\big)^2\Big)}
		{\mathbb{E}\Big(\big(
			\deltaw^{(B_i)} X
			+ W \deltax^{(B_i)}
			+ \deltaw^{(B_i)} \deltax^{(B_i)}
			\big)^2\Big)}.	\label{eq:ratio}
	\end{align}

	Thus, we obtain the lemma as stated, finishing the proof.
\end{proof}

%
\begin{figure}[t]
	\centering
	\begin{subfigure}[t]{0.49\linewidth}
		\centering
		\includegraphics[width=\linewidth, height=\linewidth]{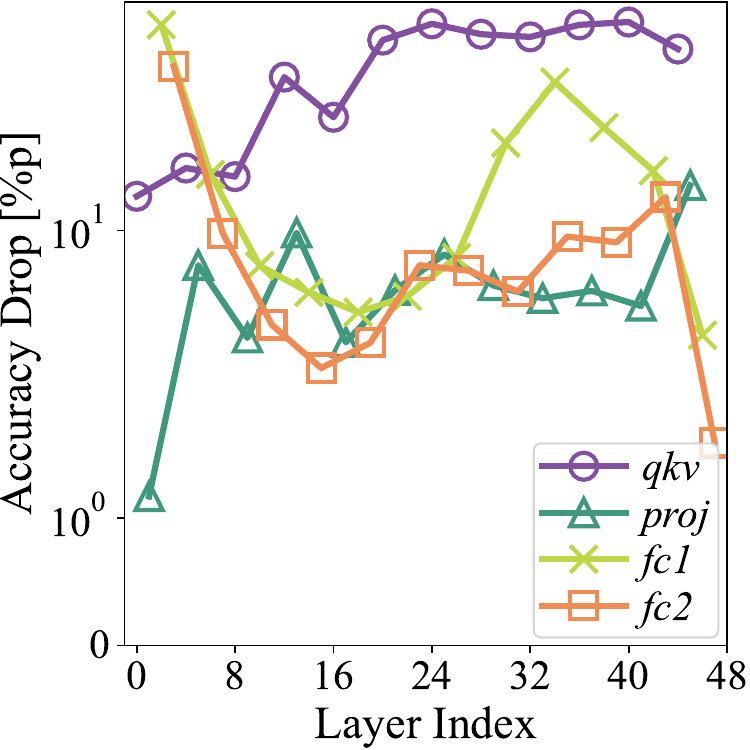}
		\caption{ViT-S}
		\label{fig:sensitivity:vit-s}
	\end{subfigure}
	\begin{subfigure}[t]{0.49\linewidth}
		\centering
		\includegraphics[width=\linewidth, height=\linewidth]{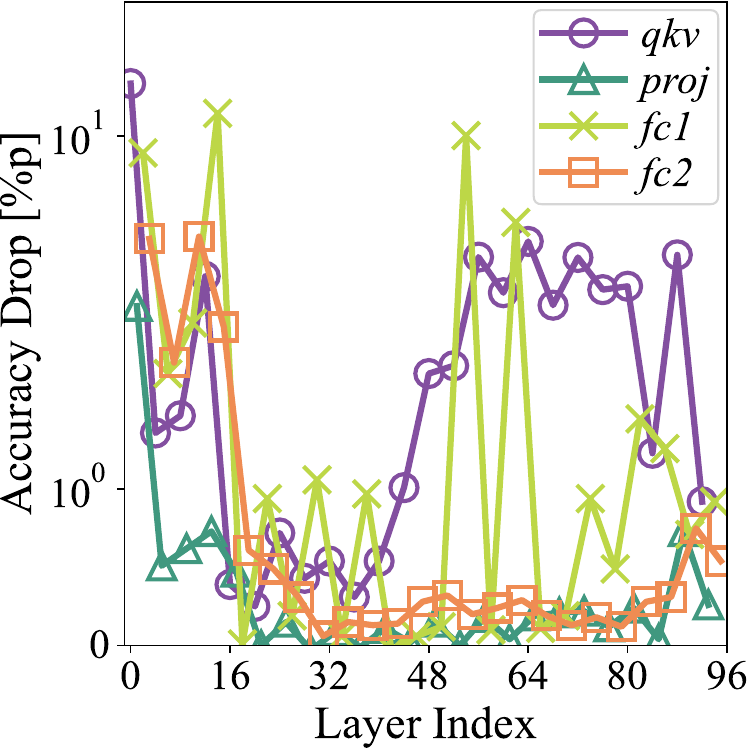}
		\caption{Swin-S}
		\label{fig:sensitivity:swin-s}
	\end{subfigure}
	\caption{
		\label{supp:fig:sensitivity}
		Impact on accuracy when quantizing a single layer of the 32-bit (a) ViT-S and (b) Swin-S models to 1-bit following AdaLog~\cite{AdaLog} while keeping the other layers unchanged.
		Sensitivity towards quantization varies significantly across  layers.
	}
\end{figure}
%

\subsection{Calculating the Relative Reconstruction Error}
	We calculate the Equation~\eqref{eq:ratio} to obtain the relative reconstruction error for each bit width $B_i$.
	Note that weight $W$ and activation $X$ are independent and both follows a Gaussian distribution.
	Since the scale of each distribution cancels out in the ratio in Equation~\eqref{eq:ratio},
	we consider only the standard normal distribution $\mathcal{N}(0, 1)$ for the Gaussian distribution.
	We assume that the clipping range of quantization is between $-3\sigma$ and $3\sigma$.
	Then, the expectation value of Equation~\eqref{eq:ratio} is expressed as follows
	since the weights and the activations are independent.
	\begin{align*}
		& \mathbb{E}\Big(\big(
			\deltaw^{(B_i)} X
			+ W \deltax^{(B_i)}
			+ \deltaw^{(B_i)} \deltax^{(B_i)}
			\big)^2\Big) \\
		&= \mathbb{E}\big((\deltaw^{(B_i)})^2\big)\mathbb{E}\big(X^2\big) + \mathbb{E}\big(W^2\big)\mathbb{E}\big((\deltax^{(B_i)})^2\big) \\
		& +\;\mathbb{E}\big((\deltaw^{(B_i)})^2\big)\mathbb{E}\big((\deltax^{(B_i)})^2\big)
			+ 2\;\mathbb{E}\big(W\deltaw^{(B_i)}\big)\mathbb{E}\big(X\deltax^{(B_i)}\big) \\
		& + 2\;\mathbb{E}\big(W\deltaw^{(B_i)}\big)\mathbb{E}\big((\deltax^{(B_i)})^2\big)
			+ 2\;\mathbb{E}\big((\deltaw^{(B_i)})^2\big)\mathbb{E}\big(X\deltax^{(B_i)}\big)
	\end{align*}
	We compute the value of each term numerically for each bit width $B_i$, as summarized in Table~\ref{tab:exp_norm}.
Then, we obtain the relative reconstruction errors for each bit width.

\section{Algorithm}
\label{appendix:algorithm}

\setlength{\textfloatsep}{16pt}
\begin{algorithm}[H]
	\caption{Quantization procedure of \method}
	\label{alg:method}
	\begin{algorithmic}[1]
		\Statex \hspace*{-\algorithmicindent}
		\textbf{Input:}
		A pre-trained full-precision ViT model $f_\theta$ with $N$ blocks,
		target bit-width $b_t$,
		sample dataset $\mathbb{D}$,
		and hyperparameters $\beta$, $\gamma$, $\mu$, and $\mathbb{B}$.
		\Statex \hspace*{-\algorithmicindent} \textbf{Output: }
		Quantized ViT model $f_{\theta'}$ within the $b_t$ bit limit
		\Statex /*\;\textit{Calculation of type-aware Fisher-based metric}\;*/
		\State Assess the initial accuracy of the pre-trained model
		\State Randomly sample $\mu$ blocks as $\mathbb{L}_{\mu}$
		\Statex //\;\textit{Obtain scaling factors $\alpha_{t}$ for each layer type $t$}
		\For{each type $t$ of layer}
		\State Calculate the average accuracy loss $A^{(t)}$ compared
		\Statex \quad\quad to the initial accuracy of the pre-trained model
		\Statex \quad\quad when quantizing a layer of type $t$ in $\mathbb{L}_{\mu}$ to $\beta$ bit
		\State Calculate the average trace $\trace(\fisher^{(t)})$ of Fisher
		\Statex \quad\quad information matrix of the quantized layer when
		\Statex \quad\quad quantizing a layer of type $t$ in $\mathbb{L}_{\mu}$ to $\beta$ bit
		\State Obtain the average accuracy loss per sensitivity
		\Statex \quad\quad $\alpha_t$ as $A^{(t)}/\trace(\fisher^{(t)})$
		\EndFor
		\For {each layer $l$ in $f_\theta$}
		\State Estimate error $\Omega_l$ with scaling factor $\alpha_t$
		\EndFor
		\Statex
		\Statex /*\;\textit{Bit allocation process}\;*/
		\Statex //\;\textit{Step 1: Initial assignment}
		\State Initialize the bit allocation $\{B_i^*\}_{i=1}^{4N}$ for each layer by solving an ILP problem 
		\State Quantize $f_\theta$ to $f_{\theta'}$ with $\{B_i^*\}_{i=1}^{4N}$
		\Statex
		\Statex //\;\textit{Step 2: Error-based iterative updates}
		\While {True}
		\For {each layer $l$ in $f_{\theta'}$}
		\State Obtain the current reconstruction loss $\loss_{recon}^{(B_i)}(l)$
		\Statex \quad\quad//\;\textit{Estimation using Table 2}
		\State Estimate the accuracy gain $L_l^{(up)}$ when
		\Statex \quad\quad\quad increasing the bit-width of $l$ by 1
		\State Estimate the accuracy loss $L_l^{(down)}$ when
		\Statex \quad\quad\quad decreasing the bit-width of $l$ by 1
		\EndFor
		\State Find the layer $u$ with the highest accuracy gain
		\State Find the layer $d$ with the lowest accuracy loss
		\State Obtain $f_{\theta^{cal}}$ by increasing the bit-width of $u$ by 1
		\Statex \quad\quad and decreasing the bit-width of $d$ by 1 from $f_{\theta'}$
		\If {$f_{\theta^{cal}}$ is less accurate than $f_{\theta'}$}
		\State \textbf{break}
		\EndIf
		\State $f_{\theta'}\leftarrow f_{\theta^{cal}}$
		\EndWhile
		\State \Return $f_{\theta'}$
	\end{algorithmic}
\end{algorithm}


We present the overall algorithm for \method in Algorithm~\ref{alg:method}.
Note that any ViT quantization technique whose quantization granularity is not larger than a layer such as RepQ-ViT~\cite{RepQ-ViT}, ERQ~\cite{ERQ}, and AdaLog~\cite{AdaLog} is applicable to quantize the model according to the bit allocation found by \method.

\section{Further Experiments}
\label{appendix:fur_experiments}
\subsection{Layer-wise Sensitivity of Different Models}
\label{supp:subsec:sensitivity}

We empirically analyze that layer-wise sensitivity to quantization varies significantly, making module-based mixed-precision approaches ineffective for optimal performance.
We investigate the layer-wise sensitivity for other ViT variants, such as ViT-S~\cite{ViT} and Swin-S~\cite{Swin}, and present the results in Figure~\ref{supp:fig:sensitivity}.
As shown in the figure, the quantization sensitivity differs greatly between layers, regardless of the model type.
This observation confirms that existing module-wise MPQ approaches fail to capture this layer-wise sensitivity, leading to suboptimal performance.

\begin{table}[t]
	\centering
	\caption{
		Runtime analysis between AdaLog~\cite{AdaLog}, VT-PTQ~\cite{VT-PTQ}, and \method (Proposed).
        hrs. and sec. indicate hours and seconds, respectively.
        \method is up to 335$\times$ faster than VT-PTQ.
	}
    \small
		\setlength{\tabcolsep}{2pt}
		\begin{tabular}{lcccc}
			\hline
			\toprule
			\textbf{Method} & \textbf{W/A} & \textbf{ViT-S} & \textbf{DeiT-S} & \textbf{Swin-S} \\
			\midrule
			AdaLog~\shortcite{AdaLog} & 3/3 & 165.3 sec. & 278.9 sec. & 648.8 sec. \\
			\midrule
			VT-PTQ~\shortcite{VT-PTQ}$^\dagger$ & 3\MP/3\MP & 23.8 hrs. & 40.2 hrs. & 93.5 hrs. \\
			\textbf{\method (Proposed)} & 3\MP/3\MP & 323.6 sec. & 484.5 sec. & 1002.6 sec. \\
			\bottomrule
			\hline
			\addlinespace[0.5ex]
			\multicolumn{5}{l}{$\dagger$: AdaLog quantization with bit allocation by VT-PTQ.}
		\end{tabular}
	\label{supp:tab:runtime}
\end{table}

\subsection{Runtime Analysis}
\label{supp:subsec:runtime}
We analyze the computational overhead of \method by measuring the quantization time of Algorithm~\ref{alg:method}.
Table~\ref{supp:tab:runtime} compares the total quantization time between AdaLog~\cite{AdaLog}, VT-PTQ~\cite{VT-PTQ}, and \method.
Note that AdaLog relies on uniform precision, while others are MPQ methods.
The results indicate that \method consistently achieves speedups of 250$\times$ or more, with maximum gains of 335$\times$ over VT-PTQ.
Compared to VT-PTQ, which assigns bits by exploring a Pareto frontier and testing hundreds of cases, \method initializes bit assignments with a type-aware Fisher-based metric and iteratively updates them based on estimated reconstruction error, significantly reducing computational overhead. 

\end{document}